\documentclass{article}

\usepackage[preprint]{neurips_2025}

\usepackage{natbib}
\usepackage[utf8]{inputenc} % allow utf-8 input
\usepackage[T1]{fontenc}    % use 8-bit T1 fonts
\usepackage{hyperref}       % hyperlinks
\usepackage{url}            % simple URL typesetting
\usepackage{booktabs}       % professional-quality tables
\usepackage{amsfonts}       % blackboard math symbols
\usepackage{nicefrac}       % compact symbols for 1/2, etc.
\usepackage{microtype}      % microtypography
\usepackage{xcolor}         % colors
\usepackage{caption}
\usepackage{tikz}
\usetikzlibrary{arrows.meta, positioning}
\usetikzlibrary{shapes.geometric}
\usepackage{subcaption}

\usepackage{amsmath}
\usepackage{amssymb}
\usepackage{mathtools}
\usepackage{amsthm}
\usepackage{stmaryrd}
\usepackage{comment}
\usepackage{centernot}
\usepackage{enumitem}

\theoremstyle{plain}
\newtheorem{theorem}{Theorem}%[section]
\newtheorem{prop}{Proposition}
\newtheorem{lemma}{Lemma}

\theoremstyle{definition}
\newtheorem{definition}{Definition}
\newtheorem{assumption}{Assumption}
\theoremstyle{remark}

\theoremstyle{remark}

\DeclareMathOperator*{\argmax}{arg\,max}

\newcommand{\ind}{\perp\!\!\!\!\!\!\perp} 

\title{Planning under Distribution Shifts \\ with Causal POMDPs}

\author{%
  Matteo Ceriscioli, Karthika Mohan \\
  Causal Intelligence and Reasoning Lab\\
 Oregon State University\\
  Corvallis, OR 97331, USA \\
  \texttt{\{ceriscim,karthika.mohan\}@oregonstate.edu} \\
}

\begin{document}

\maketitle

\begin{abstract}
In the real world, planning is often challenged by distribution shifts. As such, a model of the environment obtained under one set of conditions may no longer remain valid as the distribution of states or the environment dynamics change, which in turn causes previously learned strategies to fail. In this work, we propose a theoretical framework for planning under partial observability using Partially Observable Markov Decision Processes (POMDPs) formulated using causal knowledge. By representing shifts in the environment as interventions on this causal POMDP, the framework enables evaluating plans under hypothesized changes and actively identifying which components of the environment have been altered. We show how to maintain and update a belief over both the latent state and the underlying domain, and we prove that the value function remains piecewise linear and convex (PWLC) in this augmented belief space. Preservation of PWLC under distribution shifts has the advantage of maintaining the tractability of planning via $\alpha$-vector-based POMDP methods.
\end{abstract}

\section*{Introduction}

Planning aims to compute an optimal way for an agent to act, assuming access to a complete and correct model of the environment. In stochastic settings, this often means identifying a policy that maps information states to actions and maximizes long-term reward \citep{POMDP}. Difficulties arise once the agent is deployed in a different environment and the underlying state distribution or transition dynamics change. Distribution shifts undermine the assumption that the transition and observation processes remain fixed and accurately describe how states evolve.

%In this paper, we focus on scenarios in which the environment may be subject to distribution shifts, i.e., changes in the probability distribution of the environment state and transition dynamics. It is well documented that distribution shifts can negatively affect the performance of AI systems \cite{shimoidara, dist_shifts}. Since planning relies on complete knowledge of system dynamics, situations where the environment is dynamic and gradually changes can reduce our ability to model it accurately. In such cases, naive planning approaches can become brittle, leading to performance degradation during plan execution.

%For example, consider deploying a delivery rover whose policy and world model were obtained under in an environment with specific conditions, such as urban areas in a developed country with temperate weather. If the same rover is deployed in rural areas with snowy or sandy terrain, or with extreme temperatures, an effective planner must account for possible changes in friction, vision capabilities, mechanical stress, and the consistency of walkways and roads, changes that could be detected from the robot's observations.
To illustrate this issue,  consider a delivery rover whose policy and world model were obtained in an environment with specific conditions, such as urban areas with temperate weather. When the same rover is deployed in a different setting, relevant factors, such as friction, visibility, or mechanical stress, may change. A conventional POMDP planner cannot identify which mechanisms have shifted or why performance has degraded. A causally informed model can represent such changes as interventions on specific components, reason about their effects on the environment and the sensors, and support effective planning under the resulting distribution shift.

This paper examines whether planning remains tractable in the presence of distribution shifts. The tractability of finite-horizon POMDPs relies on the fact that their value functions are PWLC in the belief state \citep{optimalPOMDP}. When the environment is affected by an unknown shift, the resulting ambiguity in the transition dynamics introduces an additional source of uncertainty into the model. Although certain structured forms of uncertainty are known to preserve the PWLC structure \citep{osogami}, general forms of ambiguity can cause finite-horizon POMDPs to lose their PWLC properties \citep{curseAmbiguity}. We consider a causal formulation of POMDPs in which distribution shifts are represented as interventions. Under this formulation, we show that the value function retains its PWLC structure despite the presence of unknown shifts. Thus this ensures that standard $\alpha$-vector–based POMDP planning algorithms remain applicable.

Causal modeling is widely used to support robust decision making under perturbations of the data-generating process \citep{transportability, pearlCausalRevolution, causalityForML, ceriscioli2025robust}, as it offers a principled framework for analyzing how distribution shifts affect outcomes. In particular, such shifts can be naturally represented as interventions in a causal model \citep{richens2024robust, ceriscioli2025robust}, which allows systematic reasoning about their consequences.

The primary contribution of this paper is a causal framework for planning under partial observability in the presence of distribution shifts, which preserves the PWLC structure of the value function. The framework enables the evaluation of policies under specified shifts and supports the detection and identification of changes in the environment.

All proofs are included in Appendix~\ref{app:proofs}.

%This work makes the following contributions:

%\begin{enumerate}
%    \item We frame Causal POMDPs as a model for robust planning when the environment may be affected by distribution shifts.
%    \item We propose stochastic shift interventions as a family of intervention to represent distribution shifts in causal models.
%    \item We show how to find and evaluate policies on Causal POMDPs, and how to detect and identify distribution shifts while planning.
%\end{enumerate}

\section{Preliminaries}

\begin{figure*}
    \centering
    \resizebox{0.28\textwidth}{!}{
    \begin{subfigure}{0.32\textwidth}
        \centering
        \begin{tikzpicture}[
            roundnode/.style={circle, draw=black, minimum size=8.5mm, inner sep=0pt},
            squarenode/.style={rectangle, draw=black, fill=blue!20, minimum size=8.5mm, inner sep=0pt},
            trianglenode/.style={regular polygon, regular polygon sides=3, draw=black, fill=red!30, minimum size=10.5mm, inner sep=0pt, shape border rotate=0},
            diamondnode/.style={diamond, draw=black, fill=yellow!30, minimum size=11.5mm, inner sep=0pt},
            ->, >=Stealth,
            scale=0.65
            ]
        
            % Nodes
            \node[roundnode] (Yt) {$Y_{t-1}$};
            \node[roundnode, below=0.30cm of Yt] (Xt) {$X_{t-1}$};
            \node[squarenode, below=0.30cm of Xt] (Dt) {$D_{t-1}$};
            \node[diamondnode, below=0.30cm of Dt] (Ut) {$U_{t-1}$};
            
            \node[roundnode, right=0.5cm of Yt] (Yt1) {$Y_{t}$};
            \node[roundnode, below=0.30cm of Yt1] (Xt1) {$X_{t}$};
            \node[squarenode, right=0.5cm of Dt] (Dt1) {$D_{t}$};
            \node[diamondnode, below=0.30cm of Dt1] (Ut1) {$U_{t}$};
            
            \node[roundnode, right=0.5cm of Yt1] (Yt2) {$Y_{t+1}$};
            \node[roundnode, below=0.30cm of Yt2] (Xt2) {$X_{t+1}$};
            \node[squarenode, right=0.5cm of Dt1] (Dt2) {$D_{t+1}$};
            \node[diamondnode, below=0.30cm of Dt2] (Ut2) {$U_{t+1}$};
        
            % Edges
            \draw[->] (Yt) -- (Xt);
            \draw[dashed,bend right] (Yt) to [out=35,in=145] (Dt);
            \draw[bend left] (Xt) to [out=-35,in=215] (Ut);
            \draw[->] (Dt) -- (Ut);
            \draw[->] (Yt) -- (Yt1);

            \draw[->] (Yt1) -- (Xt2);
            %\draw[->] (Dt) -- (Xt1);
            %\draw[->] (Yt1) -- (Xt1);
            \draw[dashed,bend right] (Yt1) to [out=35,in=145] (Dt1);
            \draw[bend left] (Xt1) to [out=-35,in=215] (Ut1);
            \draw[->] (Dt1) -- (Ut1);

            \draw[->] (Dt1) -- (Xt2);
            \draw[->] (Yt2) -- (Xt2);
            \draw[dashed,bend right] (Yt2) to [out=35,in=145] (Dt2);
            \draw[bend left] (Xt2) to [out=-35,in=215] (Ut2);
            \draw[->] (Dt2) -- (Ut2);
        
            % Black dots and arrows
            \node[draw=none, left=0.4cm of Dt] (dots_left) {$\cdots$};
            \draw[->] (dots_left) -- (Xt);
            \node[draw=none, right=0.4cm of Xt2] (dots_right) {$\cdots$};
            \draw[->] (Dt2) -- (dots_right);
    \end{tikzpicture}
        \caption{Non-time-homogeneous \\ Causal POMDP}
        \label{fig:POMDPa}
    \end{subfigure}}%
    \hfill
    \resizebox{0.28\textwidth}{!}{
    \begin{subfigure}{0.32\textwidth}
        \centering
        \begin{tikzpicture}[
            roundnode/.style={circle, draw=black, minimum size=8mm, inner sep=0pt},
            squarenode/.style={rectangle, draw=black, fill=blue!20, minimum size=8mm, inner sep=0pt},
            trianglenode/.style={regular polygon, regular polygon sides=3, draw=black, fill=red!30, minimum size=10mm, inner sep=0pt, shape border rotate=0},
            diamondnode/.style={diamond, draw=black, fill=yellow!30, minimum size=11mm, inner sep=0pt},
            ->, >=Stealth
            ]
        
            % Nodes
            \node[roundnode] (Yt) {$Y_{t-1}$};
            \node[roundnode, below=0.30cm of Yt] (Xt) {$X_{t-1}$};
            \node[squarenode, below=0.30cm of Xt] (Dt) {$D_{t-1}$};
            \node[diamondnode, below=0.30cm of Dt] (Ut) {$U_{t-1}$};
            
            \node[roundnode, right=0.5cm of Yt] (Yt1) {$Y_{t}$};
            \node[roundnode, below=0.30cm of Yt1] (Xt1) {$X_{t}$};
            \node[squarenode, right=0.5cm of Dt] (Dt1) {$D_{t}$};
            \node[diamondnode, below=0.30cm of Dt1] (Ut1) {$U_{t}$};
            
            \node[roundnode, right=0.5cm of Yt1] (Yt2) {$Y_{t+1}$};
            \node[roundnode, below=0.30cm of Yt2] (Xt2) {$X_{t+1}$};
            \node[squarenode, right=0.5cm of Dt1] (Dt2) {$D_{t+1}$};
            \node[diamondnode, below=0.30cm of Dt2] (Ut2) {$U_{t+1}$};
        
            % Edges
            \draw[->] (Yt) -- (Xt);
            \draw[dashed,bend right] (Yt) to [out=35,in=145] (Dt);
            \draw[bend left] (Xt) to [out=-35,in=215] (Ut);
            \draw[->] (Dt) -- (Ut);
            
            \draw[->] (Dt) -- (Xt1);
            \draw[->] (Yt1) -- (Xt1);
            \draw[dashed,bend right] (Yt1) to [out=35,in=145] (Dt1);
            \draw[bend left] (Xt1) to [out=-35,in=215] (Ut1);
            \draw[->] (Dt1) -- (Ut1);

            \draw[->] (Dt1) -- (Xt2);
            \draw[->] (Yt2) -- (Xt2);
            \draw[dashed,bend right] (Yt2) to [out=35,in=145] (Dt2);
            \draw[bend left] (Xt2) to [out=-35,in=215] (Ut2);
            \draw[->] (Dt2) -- (Ut2);
        
            % Black dots and arrows
            \node[draw=none, left=0.4cm of Dt] (dots_left) {$\cdots$};
            \draw[->] (dots_left) -- (Xt);
            \node[draw=none, right=0.4cm of Xt2] (dots_right) {$\cdots$};
            \draw[->] (Dt2) -- (dots_right);
    \end{tikzpicture}
        \caption{ Time-homogeneous \\ \hspace{0.4cm}Causal POMDP}
        \label{fig:POMDPb}
    \end{subfigure}}
    \hfill
    \resizebox{0.28\textwidth}{!}{
    \begin{subfigure}{0.32\textwidth}
        \centering
        \begin{tikzpicture}[
            roundnode/.style={circle, draw=black, minimum size=8.5mm, inner sep=0pt},
            squarenode/.style={rectangle, draw=black, fill=blue!20, minimum size=8.5mm, inner sep=0pt},
            trianglenode/.style={regular polygon, regular polygon sides=3, draw=black, fill=red!30, minimum size=10.5mm, inner sep=0pt, shape border rotate=0},
            diamondnode/.style={diamond, draw=black, fill=yellow!30, minimum size=11.5mm, inner sep=0pt},
            ->, >=Stealth
            ]
        
            % Nodes
            \node[roundnode] (Yt) {$Y_{t}$};
            \node[roundnode, below=0.30cm of Yt] (Xt) {$X_{t}$};
            \node[squarenode, below=0.30cm of Xt] (Dt) {$D_{t}$};
            \node[diamondnode, below=0.30cm of Dt] (Ut) {$U_{t}$};
            \node[roundnode, right=0.5cm of Yt] (Yt1) {$Y_{t+1}$};
            \node[roundnode, below=0.30cm of Yt1] (Xt1) {$X_{t+1}$};
            %\node[squarenode, right=0.5cm of Dt] (Dt1) {$D_{t}$};
            %\node[diamondnode, below=0.30cm of Dt1] (Ut1) {$U_{t}$};
        
            % Edges
            %\draw[->] (Yt) -- (Xt);
            \draw[dashed,bend right] (Yt) to [out=35,in=145] (Dt);
            \draw[bend left] (Xt) to [out=-35,in=215] (Ut);
            \draw[->] (Dt) -- (Ut);
            \draw[->] (Yt) -- (Xt);
            \draw[->] (Dt) -- (Xt1);
            %\draw[dashed,bend right] (Yt1) to [out=35,in=145] (Dt1);
            %\draw[bend left] (Xt1) to [out=-35,in=215] (Ut1);
            %\draw[->] (Dt1) -- (Ut1);
        
            % Black dots and arrows
            %\node[draw=none, left=0.4cm of Dt] (dots_left) {$\cdots$};
            %\draw[->] (dots_left) -- (Xt);
            %\node[draw=none, right=0.4cm of Xt1] (dots_right) {$\cdots$};
            %\draw[->] (Dt1) -- (dots_right);
    \end{tikzpicture}
        \caption{Compact representation\\ \textcolor{white}{.}}
        \label{fig:POMDPc}
    \end{subfigure}}
    \caption{On the left is an example of a non–time-homogeneous causal POMDP. If the causal model is faithful, that is, every edge in the causal graph corresponds to an actual direct dependency between two variables, then any change in the graph structure across timesteps implies a change in the transition function. In the center is a CID that is compatible with a time-homogeneous causal POMDP. On the right is a compact representation of a time-homogeneous causal POMDP compatible with the CID shown in the center.}
    \label{fig:POMDP}
\end{figure*}
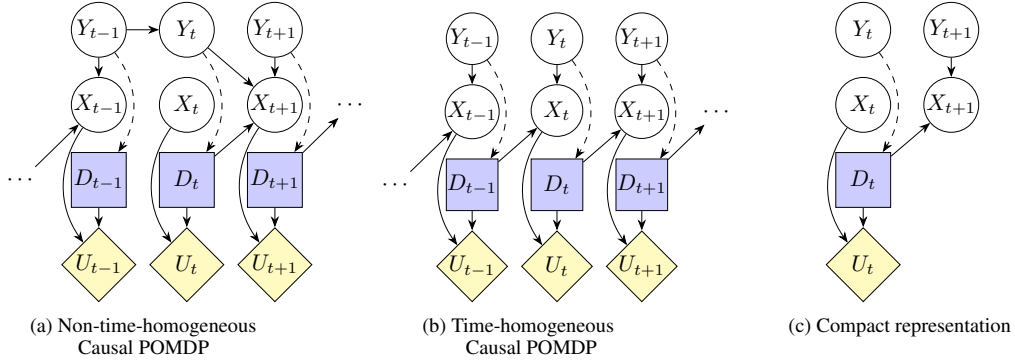

Let $X$ be a node in a graph $G$. In this paper we denote the set of parent nodes of $X$ as $Pa(X)$. Random variables are represented with upper-case letters, e.g. $X$, and the corresponding instantiations with lower-case letters, e.g. $x$. The set of observable values of a random variable $X$ is called $dom(X)$. Given a set $\mathbf X$, $\Pi(\mathbf X)$ is the set of all probability distributions over $\mathbf X$. \\ \\
\noindent \textbf{Factored POMDPs.} Partially Observable Markov Decision Processes (POMDPs) \citep{truePOMDP, POMDP} model sequential decision-making in which the agent lacks full observability of the environment. When each state in a POMDP is represented as a vector $s = (v_1, \dots, v_n)$ of instantiations of random variables $V_1, \dots, V_n$, the model is called a factored POMDP \citep{FactoredPOMDP}. This structure allows the transition function to be decomposed into sub-functions involving only subsets of the state variables.\\ \\
\noindent \textbf{Causal Influence Diagrams.} In this paper, we study factored POMDPs whose intra- and inter-temporal causal structure is expressed using Causal Influence Diagrams (CIDs) \citep{CIDHeckerman, Everitt_Carey_Langlois_Ortega_Legg_2021}. CIDs extend Causal Bayesian Networks (CBNs) \citep{causality} by incorporating decision-theoretic elements, offering a structured way to represent causal relationships while specifying what the agent observes and what it influences through its actions.
%With a slight abuse of notation, we will refer to both a node in a CBN and its corresponding variable using the same symbol. We also use $pa(X)$ when referring to an instantiation of the set of variables corresponding to nodes in $Pa(X)$. 

\begin{definition}[Causal Influence Diagram~\citep{Everitt_Carey_Langlois_Ortega_Legg_2021}]
A \textit{Causal Influence Diagram}
(CID) is a Causal Bayesian Network $M=(G=(\mathbf V,\mathbf E), P)$, where $P$ is a joint probability distribution compatible with the conditional independences encoded in the DAG $G$. The nodes in $\mathbf V$ are partitioned into decision ($\mathbf D$), utility ($\mathbf U$), and chance ($\mathbf C$) nodes, $\mathbf V = (\mathbf D, \mathbf U, \mathbf C)$. Each utility node $U_i$ is assigned a utility function $f_i : dom(Pa(U_i)) \to \mathbb{R}$. %TODO: Add an example CID
\end{definition}

 In a CID, the environment is described by the set of chance nodes $\mathbf C$. Each chance node $C\in \mathbf C$ corresponds to a random variable,  The set of decision nodes $\mathbf D$ contains all the variables which value is set by an agent, for each $D\in\mathbf D$ it is possible to assign a policy $\pi:dom(Pa(D))\to \mathbf{\mathcal{A}}$.

\section{Causal POMDPs}

There exist various causal representations of POMDPs, including models that integrate causal structure into online planning or representation learning. For example, CAR-DESPOT incorporates causal information into the AR-DESPOT planner to reason about confounding in robotic environments \citep{Cannizzaro}. Earlier work has proposed a causal POMDP model for causal representation learning and has shown its usefulness for zero-shot learning in complex tasks \citep{sontakke21a}. Other approaches use causal modeling to recover hidden causal dynamics within partially observable environments \citep{ijcai}, or employ causal abstractions to improve long-horizon reasoning \citep{causalDreamer}.

In this work, we adapt the formulation of \citet{ceriscioli2025robust}, which models a POMDP as an infinite CID unrolled over time containing a decision node and a utility node at each timestep $t$, corresponding respectively to the action and reward at time $t$.

\begin{definition}[Causal POMDP]
    A causal POMDP is a tuple $(\mathbf V,\mathbf{\mathcal{A}},\mathbf T, R, \mathbf V_o, O,\gamma)$ where:
    \begin{enumerate}
        \item $\mathbf V=(V_1,\dots,V_n)$ is the ordered set of state variables. 
        \item $\mathbf{\mathcal{A}}$ is the finite set of actions.
        \item $T: dom(\mathbf V)\times \mathbf{\mathcal{A}}\to \Pi(dom(\mathbf V))$ is the state-transition function. Given $\mathbf V^{(t)}$ the set of state variables before the transition, and $\mathbf V^{(t+1)}$ the set of state variables after the transition, it is possible to decompose $T$ as follows:
        \begin{equation}
            T(v^{(t)},a^{(t)},v^{(t+1)})=\prod_{i=1}^nT_i(v_i^{(t+1)}\mid a^{(t)},pa(v_i^{(t+1)}))
        \end{equation}
         with $Pa(V_i^{(t+1)})\subseteq \mathbf V^{(t)}\cup(\bigcup_{k<i}V^{(t+1)}_k)$. Each function $T_i$ governs the transition of the state variable $V_i$.
        %\item $\mathbf T=\{T_1^{(t)},\dots,T_n^{(t)}\}_t$ is a sequence of sets of transition functions s.t. $T_i^{(t)}: dom(V_{T_i}^{(t)})\times\mathbf{\mathcal{A}} \to \Pi(dom(V_i))$ where $\mathbf V_{T_i}^{(t)}\subseteq\mathbf V$ and $\Pi(dom(V_i))$ is the set of probability distributions over $dom(V_i)$.
        \item $R:\mathbf V_R\times \mathbf{\mathcal{A}}\to \mathbb R$ is a reward function, where $\mathbf V_{R}\subseteq\mathbf V$.
        \item $\mathbf V_o=\{V_{o,1},\dots, V_{o,p}\}$ is a set of observable variables, forming the agent's observation.
        \item $\mathbf O=\{O_1,\dots,O_p\}$ $O:dom(\mathbf V) \times \mathbf{\mathcal{A}}\times dom(\mathbf V_o)\to [0,1]$ is the set of conditional observation probabilities.
        \item $\gamma\in[0,1)$ is the discount factor.
    \end{enumerate}
\end{definition}
\noindent The set of state variables induces the set of states $\mathbf S\coloneqq dom(\mathbf V)=dom(V_1)\times\dots\times dom(V_n)$. Similarly, the set of observations is $\mathbf\Omega\coloneqq dom(\mathbf V_o)$. The agent receives a reward $R(s,a)$ after taking an action $a\in\mathbf{\mathcal{A}}$ in state $s\in \mathbf S$. \\ \\
Observable variables can be included in the state ($\mathbf V_o \subseteq \mathbf V$) without loss of expressiveness: for any causal POMDP where $\mathbf V_o\cap \mathbf V=\emptyset$, there exists an equivalent causal POMDP with $\mathbf V_o' \subseteq \mathbf V$ where transitions and rewards do not depend on $\mathbf V_o$.
\begin{assumption}\label{ass:POMDP1} 
The observable variables are also state variables, i.e. $\mathbf V_o\subseteq \mathbf V$.
\end{assumption}
When Assumption~\ref{ass:POMDP1} holds, the state fully determines the observation. Therefore, the conditional observation probability function simplifies to $O(v',a,v_o')=1$ if $v_o'$ is compatible with $v'$ and $0$ otherwise, are the new state and observation after action $a$. Since $v'$ determines $v'_o$, we can write $O(v',v'_o)$ instead of $O(v',a,v'_o)$.

Under Assumption~\ref{ass:POMDP1}, a causal POMDP $(\mathbf V,\mathbf{\mathcal{A}},\mathbf T, R, \mathbf V_o, O,\gamma)$ induces a CID with an infinite DAG $G=(\mathbf V',E)$, where $\mathbf V = \bigcup_{t=0}^\infty \mathbf V^{(t)}$ and $\mathbf V^{(t)} \coloneqq \mathbf V$ (see Figure~\ref{fig:POMDPb}). Time-homogeneity of the transition function allows a compact CID representation including only two consecutive timesteps $t$ and $t+1$, with decisions and utilities at $t$, and edges involving $D$ and $U$, between timesteps, and within $t+1$ (see Figure~\ref{fig:POMDPc}). If all state variables are observable, this reduces to a causal MDP.
%Under Assumptions \ref{ass:POMDP1} a causal POMDP $(\mathbf V,\mathbf{\mathcal{A}},\mathbf T, R, \mathbf V_o, O,\gamma)$ induces a CID composed of an infinite DAG $G=(\mathbf V',E)$ where the set of nodes $\mathbf V\coloneqq\bigcup_{t=0}^\infty \mathbf V^{(t)}$ with $\mathbf V^{(t)}\coloneqq\mathbf V$. An example of such CID is provided in Figure~\ref{fig:POMDPb}. At the same time, the time-homogeneity implied by the fact that the transition function does not depend on the timestep allows for a compact CID representation that only includes the nodes of two consecutive timesteps $t$ and $t+1$, with the decision and utility in time $t$ and the edges involving $D,U$, those between timesteps, and those within timestep $t+1$. An example of such a compact representation is provided in Figure~\ref{fig:POMDPc}. If all state variables are observable, then this results in a causal Markov Decision Process (MDP) as a special case of a causal POMDP.

\begin{comment}
    A Causal POMDP is a CID $M=(G=(\mathbf V,\mathbf E),P)$ with a single decision node $D$ and a single utility node $U$. The set of nodes $\mathbf V=\{V_1,\dots,V_n\}$ can be partitioned in $\mathbf V=\mathbf V^{(t)}\cup \mathbf V^{(t+1)}$ with $D,U\in \mathbf V^{(t)}$. The set of states $S=dom(\mathbf V^{(t+1)})$. The transition function is $T(s,a,s')=\prod_{V_i\in \mathbf V^{(t+1)}}P(V_i\mid pa(V_i))$.
\end{comment}

%\begin{assumption}\label{ass:POMDP2}
%    Time-homogeneity, i.e. the state transition probabilities do not change over time.
%\end{assumption}
Distribution shifts generally alter the state transition function, and if the distribution shift is unknown, then estimating it requires planning using a non-time-homogeneous POMDP, non-time-homogeneous causal POMDPs can also be represented with infinite CIDs, as illustrated in Figure~\ref{fig:POMDPa}.

\section{Planning under Distribution Shifts}\label{sec:uncertainty}
Being able to discern the causal relationships governing the environment allows us to infer the state dynamics under changing conditions. In a causal model, \textit{interventions} are deliberate alterations of components or mechanisms of the model.
A \textit{distribution shift} is any change in the probability distribution of a random variable. In this paper, we focus on shifts that affect the variables describing the environment in which we plan, specifically, the variables used to factorize the POMDP state in the causal POMDP model. A \textit{domain} denotes the probabilistic configuration of the environment under a particular shift, including the original one.

In this section, we develop the use of causal POMDPs for planning under potential distribution shifts. We first formalize distribution shifts as interventions within the underlying causal model. A key consequence of modeling distribution shifts as interventions is that they can be embedded directly into the agent’s belief and planning process. We proceed by describing how to evaluate a policy under a specified distribution shift. Finally, we address the problem of planning while concurrently identifying the distribution shift, showing how to update the agent’s belief over states and domains and that the resulting belief value function remains piecewise linear and convex.

\subsection{Modeling Distribution Shifts as Interventions}

As shown in previous work \citep{richens2024robust, ceriscioli2025robust}, interventions are an effective way to represent distribution shifts in a causal model.  Applying an intervention $\sigma$ to a set of variables may change the joint distribution $P(\mathbf V)$. When the model and intervention are known, it is possible to compute the updated distribution $P(\mathbf V;\sigma)$ which is called the \textit{interventional distribution}, opposed to the \textit{observational distribution} $P(\mathbf V)$, which is the one we observe when no intervention is applied.

\noindent We introduce stochastic shifts, a family of soft interventions.
\begin{definition}[Stochastic Shift Intervention]
    Let $M=(G,\Theta)$ be a CBN, and $X$ be a discrete random variable with $dom(X)=\{x_1,\dots,x_m\}$. A \textit{Stochastic Shift Intervention} $\sigma$ is an intervention associated with a matrix:
    \begin{equation}
    A_\sigma=\left(
        \begin{matrix}
p_{11} & \dots & p_{1m}\\
\vdots & \ddots & \vdots\\
p_{m1} & \dots & p_{mm}
\end{matrix}\right ) \quad\text{with $\sum_jp_{ij}=1$ for every $i$.}
    \end{equation} 
    such that when applied to a random variable $X$, its conditional distribution is updated as follows:
    \begin{equation}\label{eq:stochasticshift}
        P(X=x_j\mid pa(X);\sigma)=\sum_{i=1}^mp_{ij}P(X=x_i\mid pa(X))
    \end{equation}
\end{definition}

The condition $\sum_jp_{ij}=1$ is equivalent to requiring that each row in $A_\sigma$ sums up to one. Observe that if the probability distribution of $X$ is represented with a probability vector, e.g. $P(X|pa(X))=(P(x_1\mid pa(X)),\dots,P(x_m\mid pa(X)))^T$, then we can apply the stochastic shift intervention $\sigma$ by matrix multiplication, i.e. 
\begin{equation}
    P(X;\sigma)=A_\sigma^T P(X)
\end{equation} 
$A_\sigma$ is the identity matrix iff $\sigma$ is the identity intervention $\sigma_{id}$,  which leaves the domain and distribution unchanged.\\\\
\textbf{Example.} Let $X\sim Unif(\{1,2,3\})$ and $\sigma$ be a stochastic shift s.t. $p_{11}=p_{22}=1$, and $p_{31}=p_{32}=\frac{1}{2}$. Each time $X$ takes the value $3$, $\sigma$ remaps it to $1$ or $2$ with equal probability. So $P(X\!=\!x;\sigma)\!=\!\frac{1}{2}$ if $x\!\in\!\{1,2\}$ and $0$ if $x=3$.\\ \\
A stochastic shift intervention $\sigma$ maps probability distributions over a finite set to other distributions. The following result shows that, for any starting distribution, there exists a $\sigma$ that maps it to any target distribution.
\begin{prop}
    Given a conditional distribution $P(X\mid pa(X))$ and an arbitrary target conditional distribution $P'(X\mid pa(X))$, it is possible to define a stochastic shift intervention $\sigma$ s.t. $P'(X\mid pa(X))=P(X\mid pa(X);\sigma)$.
\end{prop}
\noindent Note that even if the original distribution $P(X\mid pa(X))$ and the shifted distribution $P(X\mid pa(X);\sigma)$ are both observed, in the general case it is not always possible to uniquely determine the distribution shift $\sigma$, as it is possible that two different shifts applied to the same distribution generate the same distribution. Suppose $X$ is the outcome of a fair coin flip, i.e., $X\sim Bern(\frac{1}{2})$, also suppose that after altering the coin we observe $Bern(\frac{3}{4})$ as the new distribution, then both
\begin{equation}
    A_\sigma=\left (\begin{matrix}
        1 & 0 \\ 0.5 & 0.5
    \end{matrix}\right)
    \qquad A_{\sigma'}=\left (\begin{matrix}
        0.5 & 0.5 \\ 1 & 0
    \end{matrix}\right)
\end{equation}
correspond to shifts that map $Bern(\frac{1}{2})$ to $Bern(\frac{3}{4})$.

\subsection{Evaluating Policies under Distribution Shifts}
% Define state-value function, state-action value function for states and beliefs

Let $M=(\mathbf V,\mathbf{\mathcal{A}},\mathbf T, R, \mathbf V_o, O,\gamma)$ be a causal POMDP, $\pi:dom(\mathbf V_o)\to \mathbf{\mathcal{A}} $ a policy, and $\sigma$ a stochastic shift intervention representing a distribution shift. When planning under partial observability, the state is generally not completely available to the agent and therefore it maintains a belief $b_S$ about the state that updates whenever it receives an observation \citep{POMDP}. Evaluating a policy consists in computing its expected return, represented by the belief value function $V^{\pi}(b)$ using the Bellman equation for causal POMDPs:

\begin{equation}\label{eq:policystate}
   V^\pi(b_S)= \sum_{s\in S}R(s,\!\pi(b_S))b_S(s)+\gamma\!\sum_{s',o'}\!O(s'\!,o')\!\sum_{s}b_S(s)\!\!\!\!\!\!\!\!\prod_{V_i'\in\mathbf V^{(t+1)}}\!\!\!\!\!\!\!\!T_i(v'_i\!\mid\! \pi(b_S), pa(V'_i))V^\pi(b_S^{\pi(b_S),o'}\!)
\end{equation}

% Here s contains some of the values of the instantiations of Pa(V_i) and s'=(v_1,...,v_n)

\noindent\textbf{Known shift $\sigma$.}  We write $V^\pi(b_S; \sigma)$ to denote the value function when the environment is affected by the shift $\sigma$. Its expression is identical to that in Equation~\ref{eq:policystate}, except that each transition term $T_i(v_i \mid \pi(b_S), pa(v_i))$  is replaced by  $T_i(v_i \mid \pi(b_S), pa(v_i); \sigma)$. Note that if $\sigma=\sigma_{id}$ then $V^{\pi}(b_S;\sigma)=V^{\pi}(b_S)$. 
Once the shift $\sigma$ is fixed and the transition functions are updated via Equation~\ref{eq:stochasticshift}, the problem of computing the value function of a causal POMDP reduces to that of a standard POMDP.
%In particular, for a known $\sigma$, the state-value function can be expressed as a linear combination of $\alpha$-vectors:

%\begin{equation}
  %  V^\pi(b_S;\sigma)=\sum_{s\in S}b(s)V^\pi(s)
%\end{equation}

%\subsection{Finding a Robust Policy for a Given Set Of Distribution Shifts}

%A reason for this is that it would be too expensive to find a policy after deployment? Maybe when in general it is not feasible to update belief about domain?

\subsection{Planning under Unknown Distribution Shifts}

Now we consider the task of planning in an environment affected by an unknown distribution shift.

\subsubsection{State and Domain Estimation.}As part of planning with causal POMDP under an unknown distribution shift, it is no longer sufficient for the agent to keep track of its belief about the state, as it also needs to keep a belief about the unknown domain. It is possible to define a prior on the state and the domain separately as $b_S(s)$ and $b_\Sigma(\sigma)$, however, since in the general case at each timestep the observation depends both on the previous state and the domain, the state belief and the domain belief become coupled and therefore it is appropriate to keep track of them using a joint belief $b(s,\sigma)$. If we have no prior knowledge we can express that as uniform priors on both the states and the domain $b(s,\sigma)=b(\sigma|s)b(s)=b_\Sigma(\sigma)b_S(s)$ where $b_S(s)\sim \text{Unif}(S)$, and $b_\Sigma(\sigma)\sim \text{Unif}(\Sigma)$.

The following proposition shows that the joint belief over states and domains admits a Bayesian update analogous to the standard POMDP filter, but extended to account for domain-dependent transition dynamics.

\begin{prop}\label{propo:belief} [State-Domain Joint Belief Update]
Let $b$ be the current joint belief over states and domains, $a$ be the action taken at time $t$, and $o'$ the observation received after performing $a$. Then the updated state-domain joint belief $b'$ is:
\begin{equation}\label{eq:belief}
\begin{split}
    b'_{o',a}(s',\sigma)=\frac{O(s',o')}{P(o'\mid a,b)}\sum_s b(s,\sigma)\!\!\!\!\!\!\prod_{V_i\in\mathbf{V^{(t+1)}}}\!\!\!\!\!T(v_i\!\mid\! pa(V_i);\sigma)
\end{split}  
\end{equation}
\end{prop}
such that $s'=(v_1,\dots,v_n)$ for $\{V_1,\dots,V_n\}=\mathbf V^{(t+1)}$.

\subsection{Preservation of PWLC under Distribution Shift}

We establish that, despite the presence of an unknown shift, the structural properties that support tractable POMDP planning remain intact.

\paragraph{Value Function for Casual POMDPs under an Unknown Shift.} To demonstrate that the value function of a causal POMDP subject to a distribution shift and planning horizon $n$ is piecewise linear and convex, and that it admits a finite vector representation analogous to that of standard POMDPs, we propose an approach similar to \citet{PortaRobot} and present a constructive proof based on the state–action value function.\\ \\ %that identifies a recursive procedure to compute the $\alpha$-vectors.
The following lemma shows that the value function under an unknown distribution shift can still be expressed as a maximum over linear functionals of the joint belief.

\begin{lemma}\label{lem:ind}
    Let $\Sigma$ be a set of stochastic shift interventions, the value function of a causal POMDP with set of states $\mathbf S$ subject to an unknown distribution shift in the set $\Sigma$ with planning horizon $n$ can be expressed as:
    \begin{equation}
        V_n(b)=\max_{\{\alpha_n^i\}_i}\sum_{s\in \mathbf S}\int_{\Sigma}\alpha^i_n(s,\sigma)b(s,\sigma)\,d\sigma
    \end{equation}
    for some functions $\alpha_n^i:S\times\Sigma\to \mathbb R$.
\end{lemma}

\noindent Using the characterization of $V_n(b)$ provided by Lemma~\ref{lem:ind}, the following theorem completes our analysis by proving that the value function under distribution shifts is still piecewise linear and convex, thereby preserving the structural form that underlies $\alpha$-vector–based planning.

\begin{theorem}
    Let $\Sigma$ be a set of stochastic shift interventions, the value function of a causal POMDP subject to an unknown distribution shift in the set $\Sigma$ with planning horizon $n$ is piecewise linear and convex in the joint state-domain belief $b$.
\end{theorem}

\noindent This confirms that distribution shifts do not compromise the representational tractability of the planning problem.

%\section{Empirical Validation}

\begin{comment}
    \subsubsection{Related work on adaptive POMDP models.} Bayesian adaptive formulations of POMDPs have been proposed to capture uncertainty in state-transition and observation models \cite{BAPOMDP}, including extensions that exploit factorized state representations \cite{FactoredBAPOMDP}. These methods treat the unknown transition dynamics as latent parameters to be inferred from data obtained through interactions and seek policies that balance learning these parameters with acting optimally. While in the factored version it is shown how to learn an appropriate graph as part of a Bayesian network encoding statistical independencies, there is no guarantee that the learned Bayesian network is causal, and therefore it is not appropriate to use this model to reason about perturbations in the data-generating process as domain shifts. In contrast, the problem address in this paper is conceptually distinct: rather than estimating fixed but unknown transition parameters, we consider the planning problem where the causal factorization and the transition model is known in a base domain, and then consider domain shifts where the underlying data-generating mechanism itself changes. Such shifts introduce a causal perturbation to the environment, requiring methods that explicitly account for alterations in the causal mechanisms governing state transitions and observations.
\end{comment}

\section*{Conclusions}
Recognizing distribution shifts as a critical source of uncertainty in planning, this work introduces a framework for decision-making under such shifts using causal POMDPs. By representing the state in a factorized form and distribution shifts as interventions on the underlying causal model, causal POMDPs enable both the evaluation of policies under specified shifts and planning in an unknown domain by maintaining a joint belief over states and domains. We show that, even under the uncertainty in the transition dynamics that arises from an unknown distribution shift, the finite-horizon value function remains piecewise linear and convex with respect to the joint state–domain belief, thereby preserving the structural properties that support $\alpha$-vector–based planning in standard POMDPs.

\newpage

\bibliography{aaai2026}

\newpage

\appendix

\section{Proof}\label{app:proofs}
\setcounter{prop}{0}
\setcounter{lemma}{0}
\setcounter{theorem}{0}
\begin{prop}
    Given a conditional distribution $P(X\mid pa(X))$ and an arbitrary target conditional distribution $P'(X\mid pa(X))$, it is possible to define a stochastic shift intervention $\sigma$ s.t. $P'(X\mid pa(X))=P(X\mid pa(X);\sigma)$.
\end{prop}
\begin{proof}
Let $dom(X)=\{x_1,\dots,x_m\}$. We define a stochastic shift interventions $\sigma$ with parameters $p_{ji}=P'(X=x_i\mid pa(X))$. Observe that for all $i\in\{1\dots,m\}$:
\begin{equation}
\begin{split}
    P(X&=x_i\mid pa(X);\sigma)=\sum_j p_{ji}P(X=x_j\mid pa(X))\\
    &=P'(X=x_i\mid pa(X))\sum_j P(X=x_j\mid pa(X))\\
    &=P'(X=x_i\mid pa(X))
\end{split}
\end{equation}
Hence, the intervention $\sigma$ transforms the original conditional distribution $P(X\mid pa(X))$ exactly into the target distribution $P'(X\mid pa(X))$, which concludes the proof.
\end{proof}

\begin{prop} [State-Domain Joint Belief Update]
Let $b$ be the current joint belief over states and domains, $a$ be the action taken at time $t$, and $o'$ the observation received after performing $a$. Then the updated state state-domain joint belief $b'$ is:
\begin{equation}
\begin{split}
    b'_{o',a}(s',\sigma)=\frac{O(s',o')}{P(o'\mid a,b)}\sum_s b(s,\sigma)\!\!\!\!\!\!\prod_{V_i\in\mathbf{V^{(t+1)}}}\!\!\!\!\!T(v_i\!\mid\! pa(V_i);\sigma)
\end{split}  
\end{equation}
\end{prop}
such that $s'=(v_1,\dots,v_n)$ for $\{V_1,\dots,V_n\}=\mathbf V^{(t+1)}$.
\begin{proof}
    By definition:
    \begin{equation}
        b'_{o',a}(s',\sigma)=P(s',\sigma\mid o',a,b) 
    \end{equation}
    where $b$ is the current belief, $a$ is the action taken at the current timestep, and $o'$ is the observation received after executing $a$. By Bayes rule:
    \begin{align}
        = \frac{P(o'\mid s',a,\sigma,b)}{P(o'\mid a,b)}P(s',\sigma\mid a, b)
    \end{align}
    Note that $o'\ind \{a,\sigma,b\}\mid s'$, so $P(o'\mid s',a,\sigma,b)=P(o'\mid s')=O(s',o')$.
    \begin{align}
        = \frac{O(s',o')}{P(o'\mid a,b)}\sum_{s\in\mathbf S}P(s'\mid s, a, \sigma, b)P(s, \sigma\mid a,b)
    \end{align}
    As $s'\ind b \mid \{s,a,\sigma\}$ and $\{s,\sigma\}\ind a\mid b$
    \begin{align}
        =&\frac{O(s',o')}{P(o'\mid a,b)}\sum_{s\in\mathbf S}P(s'\mid s, a, \sigma)P(s, \sigma\mid b) \\
        =&\frac{O(s',o')}{P(o'\mid a,b)}\sum_{s\in\mathbf S} b(s,\sigma)\!\!\!\!\prod_{V_i\in\mathbf{V^{(t+1)}}}\!\!\!\!T(v_i\mid pa(V_i);\sigma)
    \end{align}
    which corresponds to the desired expression.
\end{proof}

\begin{lemma}
    Let $\Sigma$ be a set of stochastic shift interventions, the value function of a causal POMDP with set of states $\mathbf S$ subject to an unknown distribution shift in the set $\Sigma$ with planning horizon $n$ can be expressed as:
    \begin{equation}
        V_n(b)=\max_{\{\alpha_n^i\}_i}\sum_{s\in \mathbf S}\int_{\Sigma}\alpha^i_n(s,\sigma)b(s,\sigma)\,d\sigma
    \end{equation}
    for some functions $\alpha_n^i:S\times\Sigma\to \mathbb R$.
\end{lemma}
\begin{proof}
    We prove the statement by induction on the planning horizon $n$.\\
    \textbf{Base case.} Let $n=0$, since the plan ends after the execution of a single action, then the value function corresponds to:
    \begin{equation}
        V_0(b)=\max_{a\in\mathbf{\mathcal A}} Q_0(b,a)=\max_{a\in\mathbf{\mathcal A}} \sum_{s\in \mathbf S}\int_\Sigma R(s,a)b(s,\sigma)\, d\sigma
    \end{equation}
    We define $\{\alpha_0^i(s,\sigma)\}=\{R(s,a)\}_{a\in \mathbf{\mathcal{A}}}$ and hence the value function is in the desired form.\\
    \textbf{Induction step.} Assume the statement holds for planning horizon $n-1$. Let $b_{o',a}'$ be the belief obtained by updating the belief $b$ after observing $o'$ and executing action $a$ according Equation~\ref{eq:belief}. Observe that the value function for a planning horizon $n$ can be defined recursively as:
    \begin{equation}\label{eq:vn}
    \begin{split}
        V_n(b)=\max_{a\in\mathbf{\mathcal{A}}}\{\sum_{s\in \mathbf S}\int_\Sigma b(s,\sigma)R(s,a)\, d\sigma +\gamma \sum_{o'\in\mathbf\Omega}P(o'\mid a,b)V_{n-1}(b_{o',a}')\}
    \end{split}
    \end{equation}
    by inductive hypothesis:
    \begin{equation}\label{eq:vn-1}
        V_{n-1}(b_{o',a}')=\max_{\{\alpha_{n-1}^j\}}\sum_{s'\in \mathbf S}\int_\Sigma \alpha_{n-1}^j(s',\sigma)b_{o',a}'(s',\sigma) \, d\sigma
    \end{equation}
    by substituting the expression from Equation~\ref{eq:vn-1} into Equation~\ref{eq:vn}, we obtain:
    \begin{equation}
    \begin{split}
        V_n(b)=\max_a \{&\sum_{s\in \mathbf S}\int_\Sigma R(s,a)b(s,\sigma)\, d\sigma+\\+\gamma&\sum_{o'\in\mathbf\Omega}P(o'\mid a,b)\max_{\{\alpha_{n-1}^j\}}\sum_{s'\in \mathbf S}\int_\Sigma \alpha_{n-1}^j(s',\sigma)b_{o',a}'(s',\sigma) \, d\sigma\}
    \end{split}
    \end{equation}
    by updating the belief according to Proposition~\ref{propo:belief} we get:
    \begin{equation}\label{eq:afterBelief}
        \begin{split}
            =\max_{a\in\mathbf{\mathcal{A}}} \{&\sum_{s\in \mathbf S}\int_\Sigma R(s,a)b(s,\sigma)\, d\sigma+\\
            +\gamma&\sum_{o'\in\mathbf\Omega}\max_{\{\alpha_{n-1}^j\}} \sum_{s\in \mathbf S}\int_\Sigma\left[\sum_{s'\in \mathbf S} O(s',o') P(s'\mid s,a,\sigma)\alpha_{n-1}^j(s',\sigma)\right]b(s,\sigma) \, d\sigma\}
        \end{split}
    \end{equation}
    Let $\alpha_{a,o'}^j(s,\sigma)\!\coloneqq\!\sum_{s'\in \mathbf S} \!O(s',o') P(s'\!\mid\! s,a,\sigma)\alpha_{n-1}^j(s',\sigma)$, then Equation~\ref{eq:afterBelief} can be rewritten as:
    \begin{equation}
        \begin{split}
            =\max_a \{\sum_{s\in \mathbf S}\int_\Sigma R(s,a)b(s,\sigma)\, d\sigma+\gamma\sum_{o'\in\mathbf\Omega}\max_{\{\alpha_{n-1}^j\}} \sum_{s\in \mathbf S}\int_\Sigma\alpha_{a,o'}^j(s,\sigma)b(s,\sigma) \, d\sigma\}
        \end{split}
    \end{equation}
    Let $\alpha_{a,o',b}\coloneqq \argmax_{\{\alpha_{a,o'}^j\}} \sum_{s\in \mathbf S}\int_\Sigma\alpha_{a,o'}^j(s,\sigma)b(s,\sigma) \, d\sigma$, then:
    \begin{equation}
        \begin{split}
            V_n(b)=\max_{a\in\mathbf{\mathcal{A}}}\!\left\{\sum_{s\in \mathbf S}\int_\Sigma\!\left [R(s,a)\!+\!\gamma\!\sum_{o'\in\mathbf\Omega}\!\alpha_{a,o',b}\right ]b(s,\sigma)\,d\sigma\!\right\}
        \end{split}
    \end{equation}
    Finally, we define:
    \begin{equation}
        \{\alpha_n^i\}_i\coloneqq \{R(s,a)+\gamma\sum_{o'\in\mathbf\Omega}\alpha_{a,o',b}\}_{a\in \mathbf{\mathcal{A}}}
    \end{equation}
    Then:
    \begin{equation}
        V_n(b)=\max_{\{\alpha_n^i\}_i}\sum_{s\in \mathbf S}\int_\Sigma \alpha_n^i(s,\sigma)b(s,\sigma)\, d\sigma
    \end{equation}
    proving the induction step.
\end{proof}

\begin{theorem}
    Let $\Sigma$ be a set of stochastic shift interventions, the value function of a causal POMDP subject to an unknown distribution shift in the set $\Sigma$ with planning horizon $n$ is piecewise linear and convex in the joint state-domain belief $b$:
\end{theorem}
\begin{proof}
    By Lemma~\ref{lem:ind} for every planning horizon $n$ we know that:
    \begin{equation}
        V_n(b)=\max_{\{\alpha_n^i\}_i}\sum_{s\in \mathbf S}\int_{\Sigma}\alpha_i^n(s,\sigma)b(s,\sigma)\,d\sigma
    \end{equation}
    Let $V_n^i(b)\coloneqq \sum_s\int_\Sigma \alpha_i^n(s,\sigma)b(s,\sigma)\,d\sigma$ and $b_1,b_2$ be two joint state-domain beliefs. Let $\lambda_1,\lambda_2\in \mathbb R$, then:
    \begin{equation}
    \begin{split}
        V_n^i(\lambda_1b_1+\lambda_2b_2)&=\sum_{s\in\mathbf S}\int_{\Sigma}\alpha_i^n(s,\sigma)(\lambda_1b_1(s,\sigma)+\lambda_2b_2(s,\sigma))\,d\sigma \\
        &=\lambda_1\sum_{s\in\mathbf S}\int_{\Sigma}\alpha_i^n(s,\sigma)b_1(s,\sigma)\,d\sigma + \lambda_2\sum_{s\in\mathbf S}\int_{\Sigma}\alpha_i^n(s,\sigma)b_2(s,\sigma)\,d\sigma\\
        &=\lambda_1V_n^i(b_1) + \lambda_2V_n^i(b_2)
    \end{split}
    \end{equation}
    Therefore $V_n^i(b)$ is linear in $b$. Since, for any belief $b$, the value function $V_n(b)$ is given by the $V_n^i(b)$ with the largest value, and because there is a finite number of linear functions $V_n^i(b)$, it follows that $V_n(b)$ is piecewise linear. Since linear functions are convex and $V_n(b)$ is obtained as the pointwise maximum of linear functions, $V_n(b)$ is convex as well.
\end{proof}

\end{document}